\def\year{2020}\relax
\documentclass[letterpaper]{article} 
\usepackage{aaai20}  
\usepackage{times}  
\usepackage{helvet} 
\usepackage{courier}  
\usepackage[hyphens]{url}  
\usepackage{graphicx} 
\usepackage{graphicx}  
\usepackage{amsmath, amsfonts,amsthm,algorithm,subcaption}
\newtheorem{theorem}{Theorem}
\newtheorem{remark}{Remark}
\newtheorem{lemma}{Lemma}
\usepackage[noend]{algpseudocode}

\title{Contextual-Bandit Based Personalized Recommendation with \\Time-Varying User Interests}
\author{ Xiao Xu,\textsuperscript{\rm 1} 
Fang Dong,\textsuperscript{\rm 2}\thanks{Corresponding author.}
Yanghua Li, \textsuperscript{\rm 2} 
Shaojian He,\textsuperscript{\rm 2} 
Xin Li\textsuperscript{\rm 2}\\ 
\textsuperscript{\rm 1}Cornell University, Ithaca, NY, USA\\
\textsuperscript{\rm 2}Alibaba Group, Hangzhou, Zhejiang, China\\
xx243@cornell.edu, dongfang.df@alibaba-inc.com, yichen.lyh@taobao.com, \\\{shaojian.he, xin.l\}@alibaba-inc.com }
 
\begin{document}

\maketitle

\begin{abstract}
A contextual bandit problem is studied in a highly non-stationary environment, which is ubiquitous in various recommender systems due to the time-varying interests of users. Two models with disjoint and hybrid payoffs are considered to characterize the phenomenon that users' preferences towards different items vary differently over time. In the disjoint payoff model, the reward of playing an arm is determined by an arm-specific preference vector, which is piecewise-stationary with asynchronous and distinct changes across different arms. An efficient learning algorithm that is adaptive to abrupt reward changes is proposed and theoretical regret analysis is provided to show that a sublinear scaling of regret in the time length $T$ is achieved. The algorithm is further extended to a more general setting with hybrid payoffs where the reward of playing an arm is determined by both an arm-specific preference vector and a joint coefficient vector shared by all arms. Empirical experiments are conducted on real-world datasets to verify the advantages of the proposed learning algorithms against baseline ones in both settings.
\end{abstract}

\section{Introduction}\label{sec:introduction}

Online learning has long been adopted as one of the archetypal formulations in various applications including online advertising \cite{li2010exploitation}, personalized recommendation \cite{li2010contextual}, and information retrieval \cite{yue2009interactively}. A classic framework for online learning is the multi-armed bandit (MAB) model with a set of $K$ arms (representing all possible actions) and a single player. At each time, the player chooses one of the $K$ arms to play and obtains a random reward generated from an unknown distribution specific to the chosen arm \cite{lai1985asymptotically,auer2002finite}. In order to maximize the total expected reward over a time horizon of length $T$, the learner needs to design an arm selection policy that balances an intrinsic tradeoff between exploring the unknown reward model and exploiting the current knowledge to maximize the instantaneous gain.  The performance of an arm selection policy is measured by regret, which is the expected cumulative reward loss against an omniscient player who knows the reward model and always plays the best arm.

In recent years, contextual bandits \cite{langford2007epoch,li2010contextual}, a variation of the classical MAB model, has received large attention due to its success in various online services where context information associated with either users or items is available. It has been assumed that the unknown reward model of an arm is determined by the given context. Through leveraging the context information, a number of new learning algorithms have been developed to achieve better performance compared with context-free ones in the classical setting \cite{li2010contextual,wang2016learning}.

While most existing studies on contextual bandits assume a stationary environment where the unknown reward model is fixed over time given the context information, real-world applications are usually dynamic due to the time-varying interests of users. For example, it has been observed that the click behaviors of users over different news articles evolve over time in both Google News \cite{liu2010personalized} and Yahoo News \cite{zeng2016online}. Without the capability of detecting potential changes in the underlying reward model, existing algorithms may lead to sub-optimal decisions using out-of-date observations. 

\subsection{Main Results}\label{subsec:main}
In this paper, we study a more realistic setting with non-stationary user interests under the contextual bandit framework. Specifically, we assume that the preferences of users towards items are piecewise-stationary, i.e., the reward model may undergo abrupt changes but between two consecutive change points, the model remains fixed. Furthermore, we consider asynchronous and distinct reward changes across different items, which is a common phenomenon in real applications. For example, in news recommendation, changes on the preferences of readers towards different news categories are triggered by the occurrence of related hot events, which are unlikely to happen at the same time.  In e-commerce platforms, customers' life-long interests over different products also exhibit distinct changes: a customer is more likely to purchase toys in his childhood while in adulthood, he may become more interested in sport-related products. However, the preference changes over the two categories can happen asynchronously as there may exist a time period (e.g., adolescence) when the customer likes both toys and sports.  Moreover, it is possible that the customer's preferences towards other products (e.g., snakes) remain unchanged over time. To characterize such phenomena, we consider two reward models with disjoint and hybrid payoffs as described below.

In the disjoint payoff model, we assume that the expected reward of playing an arm\footnote{An arm corresponds to an item in the recommender system. Two terms are used interchangeably thoughout the paper.} is the inner product of the given context vector and an arm-specific unknown coefficient vector, which represents the preference of the user towards the arm. The preference vector is assumed to be piecewise-stationary and the change points are different across arms. We propose an upper confidence bound (UCB) based algorithm that selects arms by estimating the unknown preference vectors from past observations. To address the challenge of time-varying interests, the algorithm adopts a change-detection procedure to identify potential changes on the preference vectors. Once a change is detected, an efficient restart is applied to re-estimate the preference vector using up-to-date observations. We provide theoretical regret analysis of the proposed algorithm and show that a sublinear scaling of regret in $T$ is achieved. We further extend the algorithm to a more general setting with hybrid payoffs. In addition to the arm-specific preference vector, the expected reward in the hybrid model also depends on a joint coefficient vector shared by all arms, which corresponds to the time-invariant component of the user interests. We conduct experiments on real-world datasets to evaluate the performance of the proposed algorithms in both settings.

\subsection{Related Work}\label{subsec:related}
Under the MAB framework, a large number of learning algorithms have been developed to balance the tradeoff between exploration and exploitation. Example algorithms include Thompson sampling \cite{thompson1933likelihood,agrawal2012analysis}, UCB \cite{lai1985asymptotically,auer2002finite}, and epsilon-greedy \cite{sutton1998reinforcement} in the classical context-free bandit setting, epoch-greedy \cite{langford2007epoch} and LinUCB \cite{li2010contextual} in the contextual bandit setting. However, those algorithms assume a stationary environment that hardly holds in real applications.

In addressing the issue of non-stationary environment, various reward models have been studied in the literature. One of the most commonly accepted models is the piecewise-stationary reward model, which allows abrupt reward changes at certain unknown time points but remains fixed between two consecutive change points. Under the piecewise-stationary assumption, the problem has been well studied in the classical context-free setting. A number of learning algorithms have been developed that adapts to the abrupt reward changes by either triggering a reset of the learning algorithm after the detected changes \cite{hartland2007change,yu2009piecewise,cao2019nearly} or applying a discount factor on past observations \cite{garivier2011upper}. Theoretical regret analysis showed that a sublinear scaling of regret in $T$ is achieved.

Within the contextual bandit setting, however, only a few recent studies have taken the issue of non-stationary environment into consideration. In \cite{hariri2015adapting}, a contextual Thompson sampling algorithm with a change detection module was proposed but theoretical regret analysis is lacking. In \cite{wu2018learning}, a hierarchical bandit algorithm was developed that detects and adapts to changes by maintaining a suite of contextual bandit models and a regret sublinear in $T$ was proved. However, the existing results assumed a uniform payoff model where all arms share a common coefficient vector representing the user interests, which fails to characterize the fact that users' preferences towards different items vary differently. Recently, a so-called context-dependent property was considered in \cite{wu2019dynamic} where arms are partitioned into change-invariant and change-sensitive ones based on their context vectors to characterize the distinct reward changes. However, the changes are not completely asynchronous across arms. A more detailed comparison between various models is discussed in the next section.
\section{Problem Formulation}
Consider a contextual bandit problem with $K$ arms and a time horizon of length $T$. At each time $t$, a recommender system observes the current player $u_t$ with a $d$-dimensional feature vector $x_{u_t}$. A subset $\mathcal{A}_t\subseteq [K]$ of arms is available for selection and each arm $a\in\mathcal{A}_t$ is associated with an $m$-dimensional feature vector $y_{a}$. The system recommends an arm $a_t$ to the user $u_t$ and observes a random reward $r_{u_t, a_t}(t)$ (i.e., clicks, ratings, etc.), which is drawn from an unknown distribution $f(\cdot;x_{u_t}, y_{a_t}, W(t))$ where $W(t)=(w_1(t),...,w_m(t))\in\mathbb{R}^{d\times m}$ is a time-varying unknown weight matrix representing the preferences of users towards items in the feature space. The conditional expectation of the reward $r_{u_t,a_t}(t)$ given the feature vectors and the weight matrix is defined as
\begin{align}\label{model:bilinear}
	\mathbb{E}[r_{u_t,a_t}(t)|x_{u_t}, y_{a_t}, W(t)] = x_{u_t}^T W(t)y_{a_t}.
\end{align}
Without loss of generality, we assume that the probability distribution of the random reward $r_{u_t,a_t}(t)$ is sub-Gaussian with parameter $\sigma$.\footnote{A random variable $Y$ with mean $\mu$ is sub-Gaussian with parameter $\sigma$ if $\mathbb{E}[e^{\lambda(Y-\mu)}]\le e^{\sigma^2\lambda^2/2},\forall \lambda\in\mathbb{R}$.} The objective is an arm selection policy $\pi$ that maximizes the expected cumulative reward over the entire time horizon, i.e., $\mathbb{E}[\sum_{t=1}^{T}r_{u_t,\pi_t}(t)]$ where $\pi_t$ is the arm selected by policy $\pi$ at time $t$. Equivalently, we may find a policy $\pi$ that minimizes the expected cumulative regret defined as the expected reward loss of policy $\pi$ against the best policy in the known model case, i.e.,
\begin{align}
	R(T)=\mathbb{E}\left[\sum_{t=1}^{T}r_{u_t,a_t^*}(t)-r_{u_t,\pi_t}(t)\right],
\end{align}
where $a_t^*$ is the arm with the largest expected reward at $t$.

In the stationary scenario where $W(t)$ is fixed over time (i.e., $W(t)\equiv W$), the above formulation is equivalent to the standard contextual bandit model with linear payoffs as studied in the literature \cite{auer2002using,chu2011contextual,agrawal2013thompson}. Specifically, let $z_{u_t,a}=\textrm{vec}(x_{u_t}y_a^T)$ be the context vector\footnote{$\textrm{vec}(\cdot)$ is the vectorization operator that concatenates columns of a matrix to a single vector.} associated with arm $a$ at time $t$ and $\beta=\textrm{vec}(W)$ be an unknown preference vector. It is clear that $\mathbb{E}[r_{u_t,a}(t)|x_{u_t}, y_{a}, W] = z_{u_t,a}^T\beta$. The unknown preference vector $\beta$ can be efficiently estimated in an online fashion at each time $t$ via ridge regression (see the LinUCB algorithm in \cite{li2010contextual}), and is applied to the reward estimation and the arm selection at time $t+1$.

In the non-stationary scenario, however, estimating $W(t)$ is in general challenging if elements of $W(t)$ vary arbitrarily: without constraints on the variation of the parameters, estimating $W(t)$ is impossible. Moreover, to characterize the fact that the preferences of users towards different items vary asynchronously and distinctly, elements of $W(t)$ should exhibit different varying patterns. However, the effects of different elements of $W(t)$ on the obtained rewards are difficult to be distinguished, which leads to the challenge of detecting unknown changes on each element from reward observations. To address the two challenges, we turn to consider approximated reward models to simplify the problem, and adopt certain assumptions on the varying patterns of the reward parameters. Specifically, we study two reward models, i.e., the \emph{disjoint payoff model} and the \emph{hybrid payoff model}.
\subsection{Disjoint Payoff Model}
In the disjoint payoff model, we let the combination of $W(t)$ and $y_a$, i.e., $\theta_a(t)=W(t)y_a$ be the unknown preference vector associated with arm $a$ at time $t$. The expected reward of recommending item $a$ to user $u$ at time $t$ is then equivalent to the inner product of $x_u$ and $\theta_a(t)$, i.e.,
\begin{align}
	\mathbb{E}[r_{u,a}(t)|x_u,\theta_a(t)] = x_u^T\theta_a(t).
\end{align}

We adopt a piecewise-stationary assumption on $\theta_a(t)$. To be specific, the time horizon is partitioned into $M_a$ stationary segments with $M_a+1$ change points $\{\nu_a^{(\ell)}\}_{\ell=0}^{M_a}$ where $\nu_a^{(0)}=0$ and $\nu_{a}^{(M_a)}=T$. Within each segment, $\theta_a(t)$ is assumed to be fixed, i.e., $\theta_a(t)\equiv\theta_a^{(\ell)}$, $\forall t\in[\nu_a^{(\ell-1)}+1, \nu_a^{(\ell)}]$, $0\le \ell\le M_a$. The sequence of changes points may be different across arms, which characterizes the fact that users' preferences towards different items may change asynchronously. 
\subsection{Hybrid Payoff Model}
In a more general model with hybrid payoffs, we further assume that $W(t)$ consists of both a time-varying component $W_v(t)$ and a time-invariant component $W_c$, i.e., $W(t)=W_v(t) + W_c$. In particular, $W_v(t)$ represents the dynamically changing preferences of users towards items and $W_c$ represents the stationary internal interests of users that are unaffected by the external environment. 

For the time-varying component $W_v(t)$, we adopt the same approximation method as the one used in the disjoint setting and define $\theta_a(t)=W_v(t)y_a$ be the arm-specific preference vector of arm $a$. For the time-invariant component, we define $\beta=\textrm{vec}(W_c)$ be the joint coefficient vector shared by all arms. It is not difficult to see that the expected reward of recommending arm $a$ to user $u$ at time $t$ satisfies that
\begin{align}
	\mathbb{E}[r_{u,a}(t)|x_u,z_{u,a},\theta_a(t),\beta] = x_u^T\theta_a(t) + z_{u,a}^T\beta,
\end{align}
where $z_{u,a}=\textrm{vec}(x_u y_a^T)$ is a $k$-dimensional ($k=d\times m$) cross-feature vector of the user-item pair. We adopt the same piecewise-stationary assumption on the arm-specific vectors $\theta_a(t)$ as that assumed in the disjoint setting, which allows asynchronous changes across different arms.

\subsection{Comparisons with Existing Models}
We first compare the two payoff models with the stationary ones in the classical contextual bandit setting. It is clear that both models are direct extensions of the stationary payoff models studied in \cite{li2010contextual} where the preference vectors $\theta_a(t),\forall a$ are assumed to be fixed over time. As discussed in the introduction section, it is more realistic to consider non-stationary preferences in real applications as users' interests are in general time-varying.

In considering the non-stationary environment within the contextual bandit setting, the majority of existing studies \cite{wu2018learning,wu2019dynamic} assumed a uniform (joint) payoff model where all arms share a common coefficient vector $\theta_u(t)$ representing the interests of user $u$. The expected reward is thus defined as
\begin{align}
\mathbb{E}[r_{u,a}(t)|y_a,\theta_u(t)] = y_a^T\theta_u(t).
\end{align} 
Notice that the uniform payoff model is another approximation of the bilinear model defined in (\ref{model:bilinear}): $\theta_u(t)$ is the combination of $x_u$ and $W(t)$, i.e., $\theta_u(t)=W^T(t)x_u$. In the literature, $\theta_u(t)$ is assumed to be piecewise-stationary to model the time-varying interests of users. The fact that users' preferences change differently towards different items is, however, not characterized.

The issue was partially addressed in \cite{wu2019dynamic} where the so-called \emph{context-dependent} property was considered. It has been assumed that the expected rewards of certain arms are insensitive to the changes of $\theta_u(t)$ (i.e., for some stationary periods $i$ and $j$, $|y_a^T\theta_u^{(i)}-y_a^T\theta_u^{(j)}|\le \Delta_L$, where $\Delta_L$ is a small constant), while the other arms are change-sensitive.  The partition of arms based on their context vectors models the distinct reward changes on different arms. However, the change points across arms are not completely asynchronous: it has been assumed in \cite{wu2019dynamic} that between any two stationary periods, there should be a sufficient number of change-sensitive arms undergo perceivable changes to distinguish the two periods. As a result, the user preferences towards a large fraction of arms change simultaneously at the change points of $\theta_u(t)$.

Moreover, we further study a general hybrid payoff model consisting of both arm-specific and joint preference vectors that correspond to the time-varying and the time-invariant interests of users respectively. To the best of our knowledge, the hybrid payoff model with dynamically changing user interests has not been studied in the literature.
\section{Piecewise-Stationary LinUCB Algorithm under the Disjoint Payoff Model}
We first consider the disjoint payoff model in this section. The key to achieving the objective of minimizing regret under the assumption of piecewise-stationary payoffs is to i) estimate the preference vectors accurately, and ii) detect the abrupt changes timely and correctly. We propose a Piecewise-Stationary LinUCB (PSLinUCB) algorithm to address the two issues.

To estimate the preference vectors, we adopt a learning structure similar to that of the LinUCB algorithm (proposed in \cite{li2010contextual} in the stationary contextual bandit setting). In particular, the unknown preference vectors $\theta_a(t),\forall a$ are estimated through ridge regression and can be updated incrementally at each time $t$. To detect the preference changes timely and correctly, the key technique adopted in the algorithm is to maintain a sliding window for each arm consisting of the most recent reward observations from the arm. If the preference vector learned from observations before the sliding window cannot accurately predict the rewards observed within the window, it is likely that the preference vector has changed. A new model should then be rebuilt based on the observations after the change point. 

To be more specific, the estimation and the change detection of the preference vector $\theta_a(t)$ of every arm $a$ can be executed independently in the disjoint payoff model. For every arm $a$, the algorithm maintains a sliding window $SW_a$ and three different models $M_a^{pre}, M_a^{cur}$, and $M_a^{cum}$. The sliding window $SW_a$ of length $\omega$ consists of the $\omega$ latest observations from arm $a$ (including the observed context vectors and the obtained rewards). $M_a^{pre}$ consists of necessary statistics for estimating the preference vector $\theta_a(t)$. It is learned from observations after the last detected change point and before the sliding window $SW_a$. Similarly, $M_a^{cur}$ with the same set of statistics is learned from observations within the sliding window, and $M_a^{cum}$ is learned from all observations from the last detected change point to the current time step. In the following subsections, we describe the details of the three models and their usage in the two key components of the PSLinUCB-Disjoint algorithm: (i) parameter estimation and arm selection, and (ii) change detection and model update. 

\subsection{Parameter Estimation and Arm Selection}
In each of the three models $M_a^{pre}, M_a^{cur}$, and $M_a^{cum}$, the preference vector $\theta_a(t)$ can be estimated by applying ridge regression to the associated set of observations. Without loss of generality, we take $M_a^{cum}$ for an example to illustrate the estimation process. Denote $\{(x_{u_t},r_{u_t,a})\}_{t\in \mathcal{I}_a^{cum}}$ as the set of observations where $\mathcal{I}_a^{cum}$ is the set of time steps when arm $a$ is played from its last detected change time (initialized to be $0$) to the current time step. $\hat{\theta}_a^{cum}$ can be estimated as $\hat{\theta}_a^{cum}=(\textbf{A}_a^{cum})^{-1}\textbf{b}_a^{cum}$ where $\textbf{A}_a^{cum}=\textbf{I}_d+\sum_{t\in\mathcal{I}_a^{cum}}x_{u_t}x_{u_t}^T$, $\textbf{I}_d$ is a $d\times d$ identity matrix, and $\textbf{b}_a^{cum}=\sum_{t\in\mathcal{I}_a^{cum}}r_{u_t,a}(t)x_{u_t}$. The statistics $\textbf{A}_a^{cum}$~and~$\textbf{b}_a^{cum}$ can be updated incrementally as described in\cite{li2010contextual}.

Based on the estimated preference vector $\hat{\theta}_a^{cum}$ of every arm $a\in\mathcal{A}_t$, we select arms according to the UCB principle to balance the tradeoff between exploration and exploitation. Similar to the LinUCB algorithm, we define a UCB index for every arm $a$ at time $t$ as $x_{u_t}^T\hat{\theta}^{cum}_a + \alpha\sqrt{x_{u_t}^T(\textbf{A}_a^{cum})^{-1}x_{u_t}}$. The arm with the greatest index is selected and the reward observations from the selected arm is used to update the corresponding models.

\subsection{Change Detection and Model Update} 
To detect potential changes on an arm $a$, we use $M_a^{pre}$ to predict the rewards of playing arm $a$ at the time steps within the sliding window. We compare the predicted rewards with the observed ones to test if the model learned from earlier data still fits the current observations. To be specific, let  $\{(x_s,r_s)\}_{s=1}^{\omega}$ be the set of observations within the sliding window. We check if $|\frac{1}{\omega}(\sum_{s=1}^{\omega}x_s^T\hat{\theta}_{a_t}^{pre}-r_s)|\ge \delta$, where $\delta$ is an input threshold.

\begin{algorithm}[t]
\caption{Piecewise-Stationary LinUCB under the Disjoint Payoff Model (PSLinUCB-Disjoint)}\label{alg:disjoint}
	\begin{algorithmic}
		\State \textbf{Input}: $\alpha>0,\omega\in\mathbb{N}^+, \delta>0$.
		\For{$t=1,2,...,T$}
			\State Observe the feature vector $x_{u_t}$ of the current user $u_t$ 
			\State and the set of available arms $\mathcal{A}_t$.
			\State //\emph{Parameter Estimation and Arm Selection}
			\For{$a\in\mathcal{A}_t$}
				\If{$a$ is new}
					\State $\textbf{A}_a^{\{pre,cur,cum\}}\leftarrow \textbf{I}_d$, $\textbf{b}_a^{\{pre,cur,cum\}}\leftarrow\textbf{0}_{d\times 1}$,
					\State $SW_a\leftarrow \emptyset$.
				\EndIf
				\State $\hat{\theta}_a^{cum}\leftarrow (\textbf{A}_a^{cum})^{-1}\textbf{b}_a^{cum}$.
				\State $p_{t,a}\leftarrow x_{u_t}^T\hat{\theta}^{cum}_a + \alpha\sqrt{x_{u_t}^T(\textbf{A}_a^{cum})^{-1}x_{u_t}}$.
			\EndFor
		\State Play $a_t=\arg\max_{a\in\mathcal{A}_t}p_{t,a}$, obtain reward $r_{u_t,a_t}(t)$.
		\State Append $(x_{u_t},r_{u_t,a_t}(t))$ to the end of $SW_{a_t}$.
		\State $\textbf{A}_{a_t}^{\{cur,cum\}}\leftarrow \textbf{A}_{a_t}^{\{cur,cum\}}+x_{u_t}x_{u_t}^T$.
		\State$\textbf{b}_{a_t}^{\{cur,cum\}}\leftarrow\textbf{b}_{a_t}^{\{cur,cum\}}+r_{u_t,a_t}(t)x_{u_t}$.
		\State //\emph{Change Detection and Model Update}
		\If{$|SW_{a_t}|\ge \omega$}
			\State $\hat{\theta}_{a_t}^{pre}\leftarrow (\textbf{A}_{a_t}^{pre})^{-1}\textbf{b}_{a_t}^{pre}$.
			\State Let $SW_{a_t}=\{(x_s,r_s)\}_{s=1}^{\omega}$.
			\If{$|\frac{1}{\omega}(\sum_{s=1}^{\omega}x_s^T\hat{\theta}_{a_t}^{pre}-r_s)|\ge \delta$}
				\State $\textbf{A}_{a_t}^{\{pre,cum\}}\leftarrow \textbf{A}_{a_t}^{cur}$, $\textbf{b}_{a_t}^{\{pre,cum\}}\leftarrow \textbf{b}_{a_t}^{cur}$,
				\State $\textbf{A}_{a_t}^{cur}\leftarrow \textbf{I}_d$, $\textbf{b}_{a_t}^{cur}\leftarrow \textbf{0}_{d\times 1}$,$SW_{a_t}\leftarrow \emptyset$.
			\Else
				\State $(x_1,r_1)\leftarrow\textrm{Popleft}(SW_{a_t})$.
				\State $\textbf{A}_{a_t}^{pre}\leftarrow \textbf{A}_{a_t}^{pre}+x_1x_1^T,\textbf{A}_{a_t}^{cur}\leftarrow \textbf{A}_{a_t}^{cur}-x_1x_1^T$
				\State $\textbf{b}_{a_t}^{pre}\leftarrow \textbf{b}_{a_t}^{pre}+r_1x_1$, $\textbf{b}_{a_t}^{cur}\leftarrow \textbf{b}_{a_t}^{cur}-r_1x_1$.
			\EndIf
		\EndIf
		\EndFor	
	\end{algorithmic}
\end{algorithm}

If a change is detected on arm $a$, i.e., the average distance between the predicted rewards and the observed ones in the sliding window exceeds the threshold, we have to restart the learning process of arm $a$ using only observations after the detected change point. Instead of re-constructing a new model without using history data, we exploit the observations within the sliding window again as a warm-start to accelerate learning. In particular, we initialize $M_a^{cum}, M_a^{pre}$, which are used for arm selection and change detection respectively, with $M_a^{cur}$, which is the model learned from the latest observations after the change point (i.e., within the sliding window). The sliding window is then emptied to collect new observations until its length reaches $\omega$ again.

If no change is detected on arm $a$, i.e., the earlier and the current reward observations follow the same model, we should keep both sets of data to enhance the estimation accuracy. Therefore, $M_a^{cum}$ keeps unchanged and the sliding window is right-shifted by one time step. Note that $M_a^{pre}$ and $M_{a}^{cur}$ should be updated accordingly after the right-shifting of $SW_a$. The detailed implementation is summarized in Algorithm \ref{alg:disjoint}. Note that the computation complexity in each time step is $O(Kd^3)$ (a finite number of matrix operations for each arm) and the memory size required for learning is $O(K(d^2+d\omega))$ (three sets of statistics and a sliding window for each arm).

\subsection{Regret Analysis}\label{subsec:regret}
In this subsection, we prove an upper bound on the regret of the proposed PSLinUCB-Disjoint algorithm. We first make several modifications on the algorithm without changing the underlying key strategies to get rid of certain technical difficulties in the theoretical analysis. 

The modification includes three steps. First, to avoid heavy dependency between the estimation and the change detection of the underlying parameters and across different time steps, the observations in the sliding-window are not re-used for initialization after a detected change. Besides, the change detection procedure only uses observations within the sliding window rather than all observations after the last detected change (note that $M^{pre}$ uses observations before the sliding window). Second, once a change is detected on an arm, the learning procedures of all arms get restarted. Finally, a round-robin exploration step is added to guarantee sufficient exploration of every arm so that the changes in the reward models can be detected timely. Due to the page limit, the details of the modified algorithm are summarized in the appendix. Based on certain mild assumptions, we prove an upper bound on regret in the following theorem.

\begin{theorem}\label{mainthm}
    With appropriate choices of the input parameters, the cumulative regret of the modified PSLinUCB algorithm under the disjoint payoff model satisfies:
    \begin{align}
        R(T)\le C_1\sqrt{TMK\omega} + C_2\sqrt{TMKd^2\log^2{T}},
    \end{align}
    where $C_1,C_2$ are constants independent of $T$ and $M$ is the number of total piecewise-stationary segments, i.e.,
	\begin{align}
		M = 1 + \sum_{t=1}^{T-1}\mathbb{I}(\theta_a(t)\neq\theta_a(t-1)\textrm{~for~some~} a\in\mathcal{A}).
	\end{align}
\end{theorem}
\begin{proof}
See the appendix.
\end{proof}
\begin{remark}
Assume that $M\ll T$. The cumulative regret achieved by the modified PSLinUCB-Disjoint algorithm has a sublinear scaling in $T$, i.e., $R(T)\sim\tilde{O}(\sqrt{T})$ where the $\tilde{O}$ notation hides the logarithmic factor. In other words, the average regret per time step diminishes to zero as $T\to\infty$.
\end{remark}
\section{Extension to the Hybrid Payoff Model}
In the hybrid payoff model, the preference of a user towards an arm $a$ is determined by both an arm-specific preference vector $\theta_a(t)$ and a joint coefficient vector $\beta$, which should be estimated simultaneously. Therefore, in addition to a sliding window $SW_a$ and three models $M_a^{pre}, M_a^{cur}$, and $M_a^{cum}$ for each arm $a$, the PSLinUCB-Hybrid algorithm maintains two global models $G^{pre}$ and $G^{cum}$ to estimate $\beta$. Specifically, $G^{pre}$ is the model learned from the observations from all arms before their sliding windows and is used for change detection. $G^{cum}$ is the model learned from the observations from all arms up to the current time step and is used for arm selection. The statistics in the two global models are obtained by applying ridge regression to the associated data. Due to the page limit, we omit the detailed theoretical derivations of ridge regression and only describe the process of updating the arm-specific and the global parameters.

\subsection{Parameter Estimation and Arm Selection} By applying ridge regression to the observed data, it can be shown that the joint coefficient vector $\hat{\beta}^{cum}$ is estimated as $\hat{\beta}^{cum}=(\textbf{A}^{cum}_{0})^{-1}\textbf{b}_{0}^{cum}$ where $\textbf{A}^{cum}_{0}$ and $\textbf{b}^{cum}_{0}$ are coupled with arm-specific parameters $\textbf{A}_{a_t}^{cum},\textbf{B}_{a_t}^{cum}$ and $\textbf{b}_{a_t}^{cum}$. Therefore, the global and the arm-specific parameters should be updated simultaneously. Specifically, $\textbf{A}^{cum}_{0}$ and $\textbf{b}^{cum}_{0}$ are initialized to $\textbf{I}_{m}, \textbf{0}_{m\times k}$ respectively and the parameters are updated as follows:
\begin{equation}\label{eq:cum_update}
\begin{aligned}
\textbf{A}_{0}^{cum}\leftarrow &\textbf{A}_0^{cum}+(\textbf{B}_{a_t}^{cum})^{T}(\textbf{A}_{a_t}^{cum})^{-1}\textbf{B}_{a_t}^{cum},\\
\textbf{b}_{0}^{cum}\leftarrow &\textbf{b}_0^{cum}+(\textbf{B}_{a_t}^{cum})^{T}(\textbf{A}_{a_t}^{cum})^{-1}\textbf{b}_{a_t}^{cum},\\
\textbf{A}_{a_t}^{cum}\leftarrow &\textbf{A}_{a_t}^{cum}+x_{u_t}x_{u_t}^T,\\
\textbf{B}_{a_t}^{cum}\leftarrow &\textbf{B}_{a_t}^{cum}+x_{u_t}z_{u_t,a_t}^T,\\
\textbf{b}_{a_t}^{cum}\leftarrow&\textbf{b}_{a_t}^{cum}+r_{u_t,a_t}(t)x_{u_t},\\
\textbf{A}_{0}^{cum}\leftarrow &\textbf{A}_0^{cum}+z_{u_t,a_t}z_{u_t,a_t}^T\\
&-(\textbf{B}_{a_t}^{cum})^{T}(\textbf{A}_{a_t}^{cum})^{-1}\textbf{B}_{a_t}^{cum},\\
\textbf{b}_{0}^{cum}\leftarrow &\textbf{b}_0^{cum}+r_{u_t,a_t}(t)z_{u_t,a_t}\\
&-(\textbf{B}_{a_t}^{cum})^{T}(\textbf{A}_{a_t}^{cum})^{-1}\textbf{b}_{a_t}^{cum}.
\end{aligned}
\end{equation}
The update procedures of $\textbf{A}_{a_t}^{cur},\textbf{B}_{a_t}^{cur}$ and $\textbf{b}_{a_t}^{cur}$ are similar to the ones of $\textbf{A}_{a_t}^{cum},\textbf{B}_{a_t}^{cum}$, and $\textbf{b}_{a_t}^{cum}$ as described above.

In the arm selection step, we follow \cite{li2010contextual} to define the UCB index of arm $a$ at time $t$ as $x_{u_t}^T\hat{\theta}_a^{cum} + z_{u_t,a}^T\hat{\beta}^{cum}+ \alpha\sqrt{s_{t,a}}$ where $\hat{\theta}_a^{cum}= (\textbf{A}_a^{cum})^{-1}(\textbf{b}_a^{cum}-\textrm{B}_{a}^{cum}\hat{\beta}^{cum})$. The exploration term $s_{t,a}=s_{t,a}^{(1)}+s_{t,a}^{(2)}+s_{t,a}^{(3)}$ is computed as follows:
\begin{equation}\label{eq:confidence}
\begin{aligned}
s_{t,a}^{(1)}=&z_{u_t,a}^T(\textbf{A}_0^{cum})^{-1}z_{u_t,a} + x_{u_t}^T(\textbf{A}_a^{cum})^{-1}x_{u_t},\\
s_{t,a}^{(2)}=&-2z_{u_t,a}^T(\textbf{A}_0^{cum})^{-1}(\textbf{B}_a^{cum})^T(\textbf{A}_a^{cum})^{-1}x_{u_t},\\
s_{t,a}^{(3)}=&x_{u_t}^T\textbf{P}({\textbf{A}_0^{cum}})^{-1}\textbf{P}^Tx_{u_t},
\end{aligned}
\end{equation}
where $\textbf{P}=(\textbf{A}_a^{cum})^{-1}\textbf{B}_a^{cum}$.
\subsection{Change Detection and Model Update}
We conduct a change detection process similar to the one adopted in PSLinUCB-Disjoint to test if the preference vector $\theta_{a_t}(t)$ of arm $a_t$ changes or not. The occurrence of a change on $a_t$ is equivalent to $a_t$ being replaced by a new arm with a different set of arm-specific parameters specified by $\textbf{A}_{a_t}^{cur},\textbf{B}_{a_t}^{cur}$, and $\textbf{b}_{a_t}^{cur}$. As a result, the global parameters $\textbf{A}_0^{cum}$ and $\textbf{b}_{0}^{cum}$ are coupled with two sets of arm-specific parameters associated with both the old and the new arm. In particular, the original arm-specific parameters (i.e., $\textbf{A}_{a_t}^{cum},\textbf{B}_{a_t}^{cum}$, and $\textbf{b}_{a_t}^{cum}$) used in estimating $\textbf{A}_0^{cum}$ and $\textbf{b}_{0}^{cum}$ should be replaced by the aggregation of the parameters corresponding to the old arm (i.e., $\textbf{A}_{a_t}^{pre},\textbf{B}_{a_t}^{pre}$, and $\textbf{b}_{a_t}^{pre}$) and the new arm (i.e., $\textbf{A}_{a_t}^{cur}, \textbf{B}_{a_t}^{cur}$, and $\textbf{b}_{a_t}^{cur}$):
\begin{equation}\label{eq:joint_cum_update}
\begin{aligned}
\textbf{A}_0^{cum}&\leftarrow\textbf{A}_0^{cum}+(\textbf{B}_{a_t}^{cum})^{T}(\textbf{A}_{a_t}^{cum})^{-1}\textbf{B}_{a_t}^{cum}\\
&~~~-(\textbf{B}_{a_t}^{pre})^{T}(\textbf{A}_{a_t}^{pre})^{-1}\textbf{B}_{a_t}^{pre}-(\textbf{B}_{a_t}^{cur})^{T}(\textbf{A}_{a_t}^{cur})^{-1}\textbf{B}_{a_t}^{cur}\\
\textbf{b}_{0}^{cum}&\leftarrow\textbf{b}_0^{cum}+(\textbf{B}_{a_t}^{cum})^{T}(\textbf{A}_{a_t}^{cum})^{-1}\textbf{b}_{a_t}^{cum},\\
&~~~-(\textbf{B}_{a_t}^{pre})^{T}(\textbf{A}_{a_t}^{pre})^{-1}\textbf{b}_{a_t}^{pre}-(\textbf{B}_{a_t}^{cur})^{T}(\textbf{A}_{a_t}^{cur})^{-1}\textbf{b}_{a_t}^{cur}.
\end{aligned}
\end{equation}
Moreover, $G^{pre}$ is re-initialized to the updated $G^{cum}$ after the detected change and the arm-specific parameters are updated in the same way with that in the disjoint payoff case.

If no change is detected on $a_t$, the updating process is similar to that in PSLinUCB-Disjoint. The detailed implementation is summarized in Algorithm \ref{alg:hybrid}.
\begin{algorithm}[t!]
\caption{Piecewise-Stationary LinUCB under Hybrid Payoff Model (PSLinUCB-Hybrid)}\label{alg:hybrid}
	\begin{algorithmic}
		\State \textbf{Input}: $\alpha>0,\omega\in\mathbb{N}^+, \delta>0, k=d\times m$.
		\State \textbf{Initialization}: $\textbf{A}_0^{pre},\textbf{A}_0^{cum}=\textbf{I}_{k}$. $\textbf{b}_0^{pre},\textbf{b}_0^{cum}=\textbf{0}_{k\times 1}$.
		\For{$t=1,2,...,T$}
		\State // \emph{Parameter Estimation and Arm Selection}
			\State Observe the feature vector $x_{u_t}$ of the current user $u_t$ 
			\State and the cross-feature $z_{u_t,a}$ for every arm $a\in\mathcal{A}_t$ .
			\State $\hat{\beta}^{cum}=(\textbf{A}_0^{cum})^{-1}\textbf{b}_0^{cum}$.
			\For{$a\in\mathcal{A}_t$}
				\If{$a$ is new}
					\State $\textbf{A}_a^{\{pre,cur,cum\}}\leftarrow \textbf{I}_d$,$\textbf{b}_a^{\{pre,cur,cum\}}\leftarrow\textbf{0}_{d\times 1}$,
					\State $\textbf{B}_a^{\{pre,cur,cum\}}\leftarrow \textbf{0}_{d\times k}$, $SW_a\leftarrow \emptyset$.
				\EndIf
				\State $\hat{\theta}_a^{cum}\leftarrow (\textbf{A}_a^{cum})^{-1}(\textbf{b}_a^{cum}-\textrm{B}_{a}^{cum}\hat{\beta}^{cum})$.
				\State $p_{t,a}\leftarrow x_{u_t}^T\hat{\theta}_a^{cum} + z_{u_t,a}^T\hat{\beta}^{cum}+ \alpha\sqrt{s_{t,a}}$.
			\EndFor		
		\State Play $a_t=\arg\max_{a\in\mathcal{A}_t}p_{t,a}$, obtain reward $r_{u_t,a_t}(t)$.
		\State Append $(x_{u_t},z_{u_t,a_t},r_{u_t,a_t}(t))$ to the end of $SW_{a_t}$.
		\State Update $\textbf{A}_0^{cum},\textbf{b}_0^{cum},\textbf{A}_{a_t}^{cum},\textbf{B}_{a_t}^{cum},\textbf{b}_{a_t}^{cum}$ using (\ref{eq:cum_update}).
		\State Update $\textbf{A}_{a_t}^{cur},\textbf{B}_{a_t}^{cur},\textbf{b}_{a_t}^{cur}$ in the same way with that in  
		\State updating $\textbf{A}_{a_t}^{cum},\textbf{B}_{a_t}^{cum},\textbf{b}_{a_t}^{cum}$ (replace $cum$ with $cur$).
		\State //\emph{Change Detection and Model Update}
		\If{$|SW_{a_t}|\ge \omega$}
			\State $\hat{\beta}^{pre}\leftarrow (\textbf{A}_{0}^{pre})^{-1}\textbf{b}_{0}^{pre}$.
			\State $\hat{\theta}_{a_t}^{pre}\leftarrow (\textbf{A}_{a_t}^{pre})^{-1}(\textbf{b}_{a_t}^{pre}-\textbf{B}_{a_t}^{pre}\hat{\beta}^{pre})$.
			\State Let $SW_{a_t}=\{(x_s,z_s,r_s)\}_{s=1}^{\omega}$.
			\If{$|\frac{1}{\omega}(\sum_{s=1}^{\omega}x_s^T\hat{\theta}_{a_t}^{pre}+z_s^T\hat{\beta}^{pre}-r_s)|\ge \delta$}
				\State Update $\textbf{A}_0^{cum},\textbf{b}_0^{cum},\textbf{A}_0^{pre},\textbf{b}_0^{pre}$ using  (\ref{eq:joint_cum_update}).
				\State  $\textbf{A}_0^{pre}\leftarrow\textbf{A}_0^{cum}, \textbf{b}_{0}^{pre}\leftarrow\textbf{b}_0^{cum}, SW_{a_t}\leftarrow \emptyset$.
				\State $\textbf{A}_{a_t}^{\{pre,cum\}}\leftarrow \textbf{A}_{a_t}^{cur}$, $\textbf{A}_{a_t}^{cur}\leftarrow \textbf{I}_d$.
				\State $\textbf{B}_{a_t}^{\{pre,cum\}}\leftarrow \textbf{B}_{a_t}^{cur}$, $\textbf{B}_{a_t}^{cur}\leftarrow \textbf{0}_{d\times k}$.
				\State $\textbf{b}_{a_t}^{\{pre,cum\}}\leftarrow \textbf{b}_{a_t}^{cur}$, $\textbf{b}_{a_t}^{cur}\leftarrow \textbf{0}_{d\times 1}$.
			\Else
				\State $(x_1,z_1,r_1)\leftarrow\textrm{Popleft}(SW_{a_t})$.
				\State Update $\textbf{A}_0^{pre},\textbf{b}_0^{pre},\textbf{A}_{a_t}^{pre}, \textbf{B}_{a_t}^{pre},\textbf{b}_{a_t}^{pre}$ according 
				\State to (\ref{eq:cum_update}) (replace $cum$ with $pre$  and
				\State $(x_{u_t},z_{u_t,a_t},r_{u_t,a_t}(t))$ with $(x_1,z_1,r_1)$).
				\State Update $\textbf{A}_{a_t}^{cur}, \textbf{B}_{a_t}^{cur},\textbf{b}_{a_t}^{cur}$ in the same way with
				\State  that in updating $\textbf{A}_{a_t}^{pre}, \textbf{B}_{a_t}^{pre},\textbf{b}_{a_t}^{pre}$  (replace
				\State $pre$ with $cur$ and operation $+$ with $-$).
			\EndIf
		\EndIf
		\EndFor	
	\end{algorithmic}

\end{algorithm}

\section{Numerical Results}\label{sec:numerical}
We use both synthetic and real-world data to evaluate the performance of the proposed learning algorithms under the disjoint and the hybrid payoff models. Due to the page limit, we only present the simulation results on two real-world datasets in this section. The results on synthetic data can be found in the appendix. 

The first real-world dataset is a collection of user-visit log information from Yahoo! front page, which is widely used for algorithm evaluation in the contextual bandit setting \cite{li2010contextual,li2011unbiased}. The Yahoo! dataset contains 45,811,883 user-visits to Yahoo Today Module in a ten-day period in May 2009. The log information of each user-visit includes a feature vector of the current user, a pool of candidate articles (arms) for recommendation associated with feature vectors, the recommended article, and the feedback from the user (click or not). It has been observed in \cite{wu2019dynamic} that the preferences of users towards different items are dynamically changing in this dataset. 

The second dataset is extracted from the Last.fm online music system, which is made available on the HetRec 2011 workshop. This dataset contains 1892 users, 17,632 artists (arms), and 92,834 user-artist listening records. Each user may assign multiple tags to the listened artists, which can be preprocessed as the context information to fit into the contextual bandit setting. Following \cite{hartland2006multi}, a non-stationary environment can be simulated. 

We compare the proposed learning algorithms with the following baselines:

\begin{enumerate}
	\item \emph{Random}: a policy that selects arms uniformly at random. 
	\item 	\emph{UCB} \cite{auer2002finite}: one of the most well-known algorithms developed in the stationary context-free bandit setting.
	\item \emph{MUCB} \cite{cao2019nearly}: an extension of UCB to the context-free setting with piecewise-stationary rewards.
	\item \emph{LinUCB} \cite{li2010contextual,chu2011contextual}: a representative algorithm for stationary contextual bandits. There are three versions of LinUCB corresponding to three different models with uniform, disjoint, and hybrid payoffs.
	\item \emph{DenBand} \cite{wu2019dynamic}: a new algorithm developed under the uniform payoff model with piecewise-stationary rewards. Under the assumption of continuous rewards with little noise, the original algorithm only compares the predicted reward at a single time step with the observed one to detect potential changes. In cases with larger noise (e.g., binary rewards), we modify the algorithm by using observations at multiple time steps for change detection.
\end{enumerate}

\subsection{Yahoo! Dataset}
\begin{figure}[t!]
	\begin{center}	
	\vspace{-.2cm}
		\includegraphics[width=.9\columnwidth]{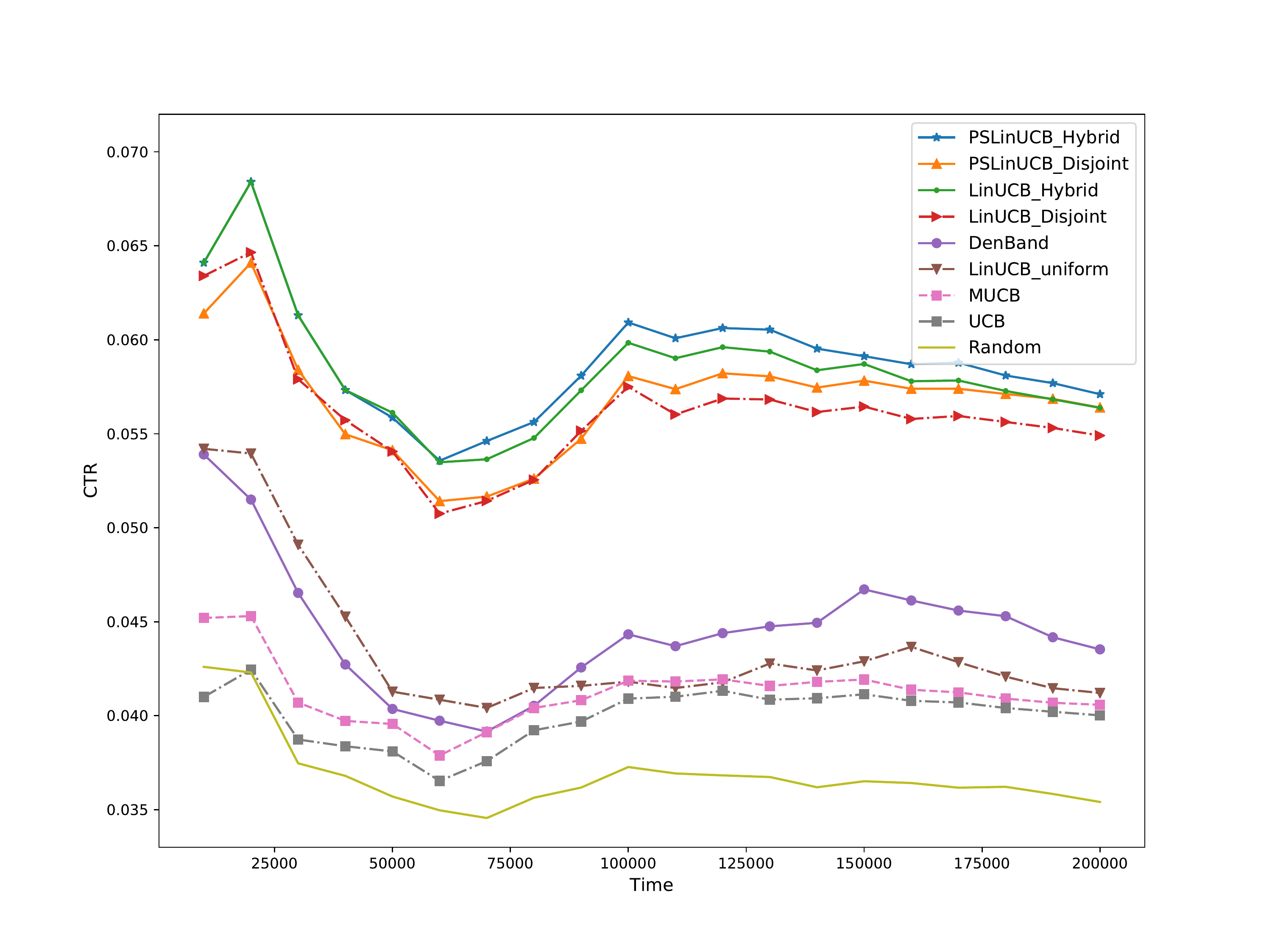}
		\vspace{-.2cm}
		\caption{Average CTR v.s. time in the Yahoo! dataset.}\label{fig:yahoo}
	\end{center}
	\vspace{-.4cm}
\end{figure}
We randomly sample a subset of data from the original dataset for testing (i.e., each user-visit is selected independently with probability $0.1$). We adopt an unbiased offline evaluation method proposed in \cite{li2010contextual,li2011unbiased} to evaluate the online performance of the proposed learning algorithms and the baseline ones. 

In Figure \ref{fig:yahoo}, we report the average Click-Through-Rate~(CTR) of different algorithms versus time. We first observe that algorithms exploiting the context information (i.e., PSLinUCB, LinUCB, and DenBand) outperform context-free ones (i.e., UCB and MUCB). This observation is rather intuitive since context vectors provide significant side information on the preferences of users towards items. In addition, under each reward model (i.e., classical context-free bandits and contextual bandits with uniform, disjoint, and hybrid payoffs), the algorithm that adapts to reward changes outperforms the one that does not (i.e., MUCB v.s. UCB, DenBand v.s. LinUCB-uniform, PSLinUCB-Disjoint v.s. LinUCB-Disjoint, and PSLinUCB-Hybrid v.s. LinUCB-Hybrid). In particular, PSLinUCB-Disjoint achieves a performance gain of $2.7\%$ ($2.9\%$ at the peak) against LinUCB-Disjoint and PSLinUCB-Hybrid achieves an improvement of $1.3\%$ ($2\%$ at the peak) against LinUCB-Hybrid (see the appendix for details). The comparison results verify the assumption that users' interests are dynamically changing and should be taken into consideration in learning. 

Moreover, within the contextual bandit setting, algorithms developed under the hybrid payoff model (i.e., PSLinUCB-Hybrid and LinUCB-Hybrid) or the disjoint payoff model (i.e., PSLinUCB-Disjoint and LinUCB-Disjoint) achieve better performance compared with the ones developed under the uniform payoff model (i.e., DenBand and LinUCB-Uniform). This is because the uniform payoff model fails to exploit the personalized interests of different users. An alternative approach is to learn the preferences of every user individually. However, the amount of data associated with a single user is rather limited. Furthermore, the performance gain of PSLinUCB over DenBand ($31.2\%$ under the hybrid model and $29.5\%$ under the disjoint model) verifies the fact that users' preferences towards different items vary differently. We also conduct experiments to anaylze the sensitivity of the proposed algorithms to the hyper-parameters. Due to the page limit, we leave the results in the appendix.

\subsection{LastFM Dataset}
Given that the original LastFM dataset dose not provide context vectors of neither users nor items, we first preprocess the dataset to fit into the contextual bandit setting. Specifically, following the settings in \cite{cesa2013gang,wu2019dynamic}, we treat the `listened artists' of each user as positive feedback. For each artist, we use its associated tags to create a TF-IDF feature vector and then apply PCA to reduce the dimension to 10. For each user, we adopt a method similar to the one used in \cite{li2010contextual} to generate a feature vector: we use matrix factorization to obtain a raw feature vector and then apply the K-means method to group users into 10 clusters. The final user feature is a 10-dimensional vector corresponding to the soft-membership of the user in the 10 clusters (computed with a Gaussian kernel and then normalized). In the experiment, we only consider artists that have been listened by at least 100 users and we follow \cite{wu2018learning} to generate the log data. The results are presented in Figure \ref{fig:lastfm} and similar conclusions with those in the experiment on the Yahoo! dataset can be drawn. In particular, PSLinUCB-Disjoint achieves a performance gain of $2\%$ against LinUCB-Disjoint and PSLinUCB-Hybrid achieves a performance gain of $2.4\%$ against LinUCB-Hybrid, which again verify the advantages of the proposed algorithms.
\begin{figure}[t!]
	\begin{center}
		\vspace{-.2cm}
		\includegraphics[width=.9\columnwidth]{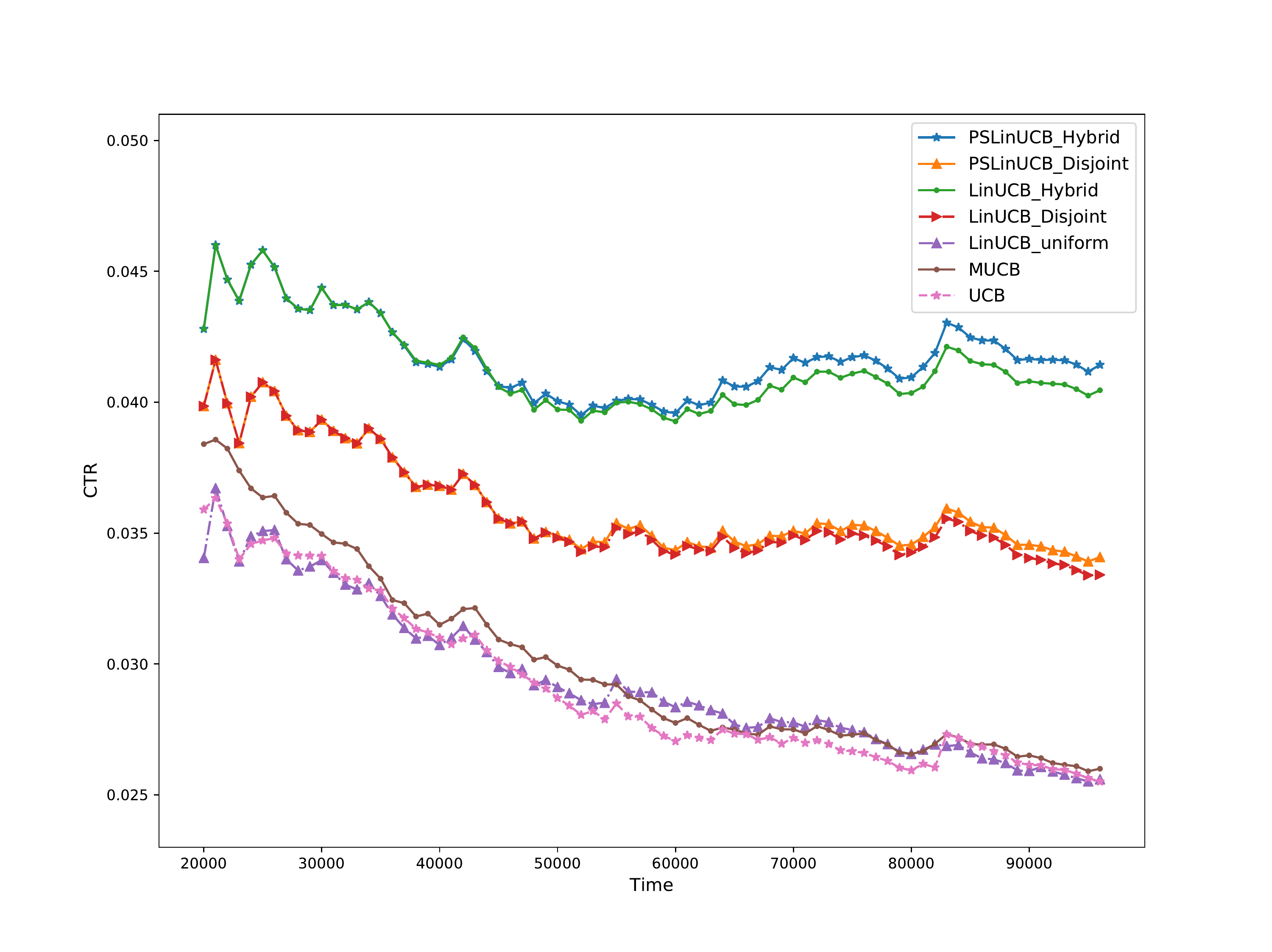}
		\vspace{-.2cm}
		\caption{Average CTR v.s. time in the LastFM dataset.}\label{fig:lastfm}
	\end{center}
	\vspace{-.4cm}
\end{figure}

\section{Conclusions and Future Work}
We studied a contextual bandit problem for personalized recommendation in a non-stationary environment. To characterize the fact that users' interests towards different items vary asynchronously and distinctly, two models with disjoint and hybrid piecewise-stationary payoffs were considered. Efficient learning algorithms were developed under both models and theoretical analysis validating a vanishing per-time regret was provided under the disjoint payoff model. Numerical results on real-world datasets verified the advantages of the proposed learning algorithms against baseline ones.

Several issues in this work ask for future studies. First, theoretical regret analysis under the hybrid payoff model is still lacking. Moreover, one limitation of the proposed algorithm is that estimating a preference vector for every arm is costly in computation and memory, especially when the number of arms is extremely large. A potential solution is to cluster similar arms into groups and collectively learn the preferences of users towards arms within the same group.

\bibliographystyle{aaai}

\bibliography{References}

\newpage
\onecolumn
\appendix
\section{Appendices}

\subsection{Modified PSLinUCB-Disjoint Algorithm}
To simplify the algorithm design and regret analysis, we assume that the candidate arm set at each time is fixed, i.e., $\mathcal{A}_t=\mathcal{A}$. The details of the modified PSLinUCB-Disjoint algorithm are summarized below in Algorithm \ref{alg:disjoin_modify}.
\begin{center}
\begin{algorithm}[h!]
\begin{center}
\caption{Modified PSLinUCB-Disjoint}\label{alg:disjoin_modify}
	\begin{algorithmic}
		\State \textbf{Input}: $\alpha>0,\omega\in\mathbb{N}^+, b,c>0, \gamma>0$, and the arm set $\mathcal{A}$.
		\State \textbf{Initialization}: $\tau\leftarrow 0, \textbf{A}_a^{cum}\leftarrow \textbf{I}_d,\textbf{b}_a^{cum}\leftarrow\textbf{0}_{d\times 1}$, $\textrm{SW}(a)\leftarrow \emptyset,\forall a\in\mathcal{A}$.
		\For{$t=1,2,...,T$} do
			\State //\emph{Round-Robin Exploration}
			\State Let $a=(t-\tau)~\textrm{mod}~\lfloor K/\gamma\rfloor$.
			\If{$a\le K$}
				\State Play arm $a_t=a$.
			\Else
			\vspace{.1cm}
			\State //\emph{Parameter Estimation and Arm Selection}
			\vspace{.1cm}
			\State Observe the feature vector $x_{u_t}$ of the current user $u_t$.
			\vspace{.1cm}
			\For{$a\in\mathcal{A}$} do
			\vspace{.1cm}
				\State $\hat{\theta}_a\leftarrow (\textbf{A}_a^{cum})^{-1}\textbf{b}_a^{cum}$.
				\vspace{.1cm}
				\State $p_{t,a}\leftarrow x_{u_t}^T\hat{\theta}_a + \alpha\sqrt{x_{u_t}^T(\textbf{A}_a^{cum})^{-1}x_{u_t}}$.
			\EndFor
		
			\State Play $a_t=\arg\max_{a\in\mathcal{A}_t}p_{t,a}$, obtain reward $r_{u_t,a_t}$.
			\vspace{.1cm}
			\State Append $(x_{u_t},r_{u_t,a_t}(t))$ to the end of $SW(a_t)$.
			\EndIf
			\vspace{.1cm}

			\State  $\textbf{A}_{a_t}^{cum}\leftarrow \textbf{A}_{a_t}^{cum}+x_{u_t}x_{u_t}^T$.
			\vspace{.1cm}
			\State $\textbf{b}_{a_t}^{cum}\leftarrow\textbf{b}_{a_t}^{cum}+r_{u_t,a_t}x_{u_t}$.
			\vspace{.1cm}
			\State //\emph{Change Detection and Model Update}
			\vspace{.1cm}
			\If{$|SW_{a_t}|\ge \omega$}
			\vspace{.1cm}
				\State Let $SW_{a_t}=\{(x_s,r_s)\}_{s=1}^{\omega}$.
				\vspace{.1cm}
				\State $ \textbf{A}_{a_t}^{pre}=\sum_{s=1}^{\lfloor\omega/2\rfloor}x_sx_s^T,~\textbf{b}_{a_t}^{pre}=\sum_{s=1}^{\lfloor\omega/2\rfloor}r_sx_s^T$.
				\vspace{.1cm}
				\State $\hat{\theta}_{a_t}^{pre}\leftarrow (\textbf{A}_{a_t}^{pre})^{-1}\textbf{b}_{a_t}^{pre}$.
			\vspace{.1cm}
				
				\If{$|\frac{2}{\omega}(\sum_{s=\lfloor\omega/2\rfloor+1}^{\omega}x_s^T\hat{\theta}_{a_t}^{pre}-r_s)|\ge b+c$}
				\vspace{.1cm}
					\State $\forall a\in\mathcal{A}:\textbf{A}_a^{cum}\leftarrow \textbf{I}_d,\textbf{b}_a^{cum}\leftarrow\textbf{0}_{d\times 1},SW_a\leftarrow\emptyset$.
\vspace{.1cm}					
					\State $\tau\leftarrow t$.
				\EndIf
			\EndIf
		\EndFor	
	\end{algorithmic}
	\end{center}
\end{algorithm}
\end{center}

\subsection{Proof of Theorem 1}
Before we present the proof on the regret upper bound, we first introduce some notations used in the analysis.
Let $\{\nu_i\}_{i=0}^{M}$ be the change-points where $\nu_0=0, \nu_M = T$. Define $L=\omega\lceil K/\gamma\rceil$ where $\omega,\gamma$ are input parameters of the modified PSLinUCB policy. Let $\Delta_a^{(i)}(x)$ be the amplitude of the preference change from user $x$ to arm $a$ at the $i$-th change point,~i.e.,
\begin{align}
	\Delta_a^{(i)}(x)=\left|x^T\theta_a(\nu_i+1)-x^T\theta_a(\nu_i)\right|.
\end{align}

Without loss of generality, we assume that the sub-Gaussian parameter $\sigma$ in the distribution of the random reward is $1$. We further assume that $||\theta_a(t)||_2\le 1,||x_{u_t}||_2\le 1,\forall t,\forall a\in\mathcal{A} $.
The proof of Theorem 1 is based on three key lemmas as presented below. To avoid breaking the logic flow, we leave the detailed proofs of the three lemmas to the next three appendices.

We first consider a stationary scenario where the reward model, i.e., $\theta_a(t)$ is fixed for all $a\in\mathcal{A}$.
\begin{lemma}
	Consider a stationary scenario with $M=1$, consider $\delta_0\in(0,1)$ and $\alpha>\sqrt{2d\log\frac{T}{\delta_0}}$, the expected cumulative regret of the modified PSLinUCB algorithm is upper bounded as follows:
	\begin{equation}
	\begin{aligned}
		\mathbb{E}[R(T)]\le &T\mathbb{P}(\tau_1\le T) + (\delta_0 + \gamma) T + K\\
		&+2\alpha\sqrt{2TdK\log{\frac{T}{d}}},
	\end{aligned}
	\end{equation}
	where $\tau_1$ is the first detection time. 
\end{lemma}

Second, we upper bound the probability of raising false alarms, i.e., changes are detected in the stationary environment.
\begin{lemma}
Consider a stationary scenario with $M=1$, consider $\delta_1\in(0,1)$,  the probability of false alarm is upper bounded as follows:
\begin{align}
	\mathbb{P}(\tau_1\le T)\le \sum_{a\in\mathcal{A}}\omega(1-(1-(2e^{-\omega c^2}+\delta_1))^{T/\omega}),
\end{align}
if the threshold $b$ satisfies (\ref{bat}) for all $a\in\mathcal{A}$. Let $\delta_1=1/(2T^2)$ and assume $c\ge \sqrt{\frac{2}{\omega}\log (2T)}$, we have
\begin{align}
	\mathbb{P}(\tau_1\le T)\le KT^{-1}.
\end{align}
\end{lemma}

We further upper bound the probability of a late detection.
\begin{lemma}
Consider a piecewise-stationary scenario with $M\ge 2$. Assume that $\Delta_{a}^{(1)}(x)\ge b+c$ for some $a\in\mathcal{A}$ and for any $x$.  Suppose $\nu_2-\nu_1> L$ and $\nu_1>L/2$. Then we have
\begin{align}
\mathbb{P}(\tau_1>\nu_1+L/2)\le 2T^{-2}.
\end{align}
\end{lemma}
Theorem 1 can be proved based on the above three lemmas and properties of the restart process of the algorithm. Specifically, define events $F_i=\{\tau_i\ge \nu_i\},1\le i\le M-1$ and $D_i=\{\tau_i<\nu_i+L/2\},1\le i\le M-2$, $D_{M-1}=\{\tau_{M-1}\le T\}$. Then we have
\begin{equation}
\begin{aligned}
	&\mathbb{E}[R(T)]\le\mathbb{E}[R(T)\mathbb{I}(F_1)]+T(1-\mathbb{P}(F_1))\\
	\le&\mathbb{E}[R(\nu_1)\mathbb{I}(F_1)] + \mathbb{E}[R(T)-R(\nu_1)] + K\\
	\le &(\delta_0+\gamma)\nu_1 + 2\alpha\sqrt{2\nu_1dK\log{\frac{\nu_1}{d}}}+2K +\mathbb{E}[R(T)-R(\nu_1)]
\end{aligned}
\end{equation}
Note that the first inequality follows from Lemma 2 on bounding the probability of false alarm in the first fist stationary segment $[0,\nu_1]$ provided that $b$ satisfies (\ref{bat}) and $c=\sqrt{\frac{2}{\omega}\log(2T)}$. The second inequality follows from Lemma 1 on $[0,\nu_1]$. The next step is to bound $\mathbb{E}[R(T)-R(\nu_1)]$, which satisfies
\begin{equation}
	\begin{aligned}
		&\mathbb{E}[R(T)-R(\nu_1)]\\
		\le& \mathbb{E}[R(T)-R(\nu_1)|F_1D_1] + T(1-\mathbb{P}(F_1D_1))\\
		=& \mathbb{E}[R(T)-R(\nu_1)|F_1D_1] + T(\mathbb{P}(\bar{F}_1D_1)+\mathbb{P}(F_1\bar{D}_1)+\mathbb{P}(\bar{F}_1\bar{D}_1))\\
		\le& \mathbb{E}[R(T)-R(\tau_1)|F_1D_1] + \mathbb{E}[R(\tau_1)-R(\nu_1)|F_1D_1] + 2K\\
		\le&\tilde{\mathbb{E}}[R(T-\tau_1)] + \mathbb{E}[\tau_1-\nu_1|F_1D_1] + 2K\\
		\le&\tilde{\mathbb{E}}[R(T-\tau_1)] + \omega\lceil K/\gamma\rceil + 2K,
	\end{aligned}
\end{equation}
where the second inequality holds due to the following facts 
\begin{itemize}
\item $\mathbb{P}(\bar{F}_1D_1)=\mathbb{P}(\bar{F}_1)\le KT^{-1}$ according to Lemma 2, provided that $b$ satisfies (\ref{bat}) and $c=\sqrt{\frac{2}{\omega}\log(2T)}$;
\item $\mathbb{P}(F_1\bar{D}_1)=\mathbb{P}(\bar{D}_1)\le 2T^{-2}$ according to Lemma 3;
\item $\mathbb{P}(\bar{F}_1\bar{D}_1)=0$ since $\bar{F}_1$ and $\bar{D}_1$ cannot happen simultaneously.
\end{itemize}
The third inequality holds due to the fact that the learning process is restarted once a change is detected and $\tilde{\mathbb{E}}$ is the expectation taken over the random process induced by the learning algorithm after the first detected change time.

Finally, if we recursively upper bound $\tilde{\mathbb{E}}[R(T-\nu_1)]$ by the same arguments as above and repeat the process for $M-1$ times, we have
\begin{equation}
\begin{aligned}
\mathbb{E}[R(T)]\le& (\delta_0+\gamma)T + \sum_{i=1}^{M}2\alpha\sqrt{2\nu_i dK\log{\frac{\nu_i}{d}}} \\
&+ 4KM + \omega M \lceil K/\gamma\rceil.
\end{aligned}
\end{equation}
Let $\delta_0=1/T$, $\gamma=\sqrt{\frac{KM\omega}{T}}$, $\alpha>\sqrt{2d\log \frac{T}{\delta_0}}$, and apply Cauchy-Schwatz inequality to the second term, we can obtain
\begin{align}
\mathbb{E}[R(T)]\le C_1 \sqrt{TMK\omega} + C_2\sqrt{TMKd^2\log^2{T}} .
\end{align}

\subsection{Proof of Lemma 1}

Let $\mathbb{I}(\cdot)$ be the indicator function and $R_{a_t}$ be the one-step regret at time $t$ when the algorithm plays arm $a_t$. The expected cumulative regret can be partitioned as follows:
\begin{equation}
\begin{aligned}
	\mathbb{E}[R(T)]&=\mathbb{E}[R(T)\mathbb{I}(\tau_1\le T)] + \mathbb{E}[R(T)\mathbb{I}(\tau_1>T)]\\
	&\le T\cdot\mathbb{P}(\tau_1\le T) + \mathbb{E}[R(T)\mathbb{I}(\tau_1> T)]\\
	&\le T\cdot\mathbb{P}(\tau_1\le T) + \sum_{t=1}^{T}\mathbb{E}[R_{a_t}\mathbb{I}(\tau_1>T, a_t\textrm{ is random selected})] \\
	&~~~+  \sum_{t=1}^{T}\mathbb{E}[R_{a_t}\mathbb{I}(\tau_1>T, a_t\textrm{ is selected by UCB index})].\\
\end{aligned}
\end{equation}	
According to the algorithm, it is not difficult to see that the second term on the RHS of the above inequality satisfies
\begin{equation}
\begin{aligned}
	&\sum_{t=1}^{T}\mathbb{E}[R_{a_t}\mathbb{I}(\tau_1>T, a_t\textrm{ is random selected})] \le K\cdot \left\lceil\frac{T\gamma}{K}\right\rceil\le K+T\gamma.
\end{aligned}
\end{equation}
For the last term, we have:
\begin{equation}\label{lm1eq1}
\begin{aligned}
	&\sum_{t=1}^{T}\mathbb{E}[R_{a_t}\mathbb{I}(\tau_1>T, a_t\textrm{ is selected by UCB index})]\\
	\le&\sum_{t=1}^{T}\mathbb{E}[(r_{a_t^*}-r_{a_t})\mathbb{I}(\forall a\in\mathcal{A}, \textrm{no change detected up to time $t-1$},a_t\textrm{ is selected by UCB index})]\\
	=&\sum_{t=1}^{T}(x_{u_t}^T\theta_{a_t^*}-x_{u_t}^T\theta_{a_t})\mathbb{I}(\forall a\in\mathcal{A}, \textrm{no change detected up to time $t-1$},a_t\textrm{ is selected by UCB index})
\end{aligned}
\end{equation}
Note that if no change has been detected up to time $t-1$, the estimation of $\theta_a,\forall a\in\mathcal{A}$ has not been restarted and thus, $\hat{\theta}_a$ is calculated based on all past observations. Thus, according to the algorithm, the RHS of (\ref{lm1eq1}) is upper bounded by
\begin{equation}\label{lm1eq2}
\begin{aligned}
	&\sum_{t=1}^{T}(x_{u_t}^T\theta_{a_t^*}-x_{u_t}^T\theta_{a_t})\mathbb{I}(\forall a\in\mathcal{A}, \textrm{no change detected up to time $t-1$},a_t\textrm{ is selected by UCB index})\\
	\le&\sum_{t=1}^{T}(x_{u_t}^T\hat{\theta}_{a_t^*} + ||\hat{\theta}_{a_t^*}-\theta_{a_t^*}||_{A_{a_t^*}(t-1)}\cdot ||x_{u_t}||_{A_{a_t^*}^{-1}(t-1)}-x_{u_t}^T\theta_{a_t})\mathbb{I}(a_t\textrm{ is selected by UCB index})\\
	\le &\sum_{t=1}^{T}x_{u_t}^T\hat{\theta}_{a_t} + \alpha||x_{u_t}||_{A_{a_t}^{-1}(t-1)}-x_{u_t}^T\theta_{a_t}\\
	\le & \sum_{t=1}^{T}2\alpha ||x_{u_t}||_{A_{a_t}^{-1}(t-1)}
\end{aligned}
\end{equation} 
where $A_{a}(t-1)=\sum_{\tau=1}^{t-1}\mathbb{I}(a_\tau=a)x_{u_\tau}x_{u_\tau}^T$ and $||x||_{A}=\sqrt{x^TAx}$. The first inequality simply follows from Lemma 2 in \cite{guo2019adalinucb}. By selecting $\alpha>||\hat{\theta}_a-\theta_a||_{A_{a}(t-1)}$,$\forall a\in\mathcal{A}$ and $t$, the second inequality follows from the fact that the UCB index of $a_t$ is greater than $a_t^*$ at time $t$. The last inequality also holds according to Lemma 2 in \cite{guo2019adalinucb} and the selection of $\alpha$. It has been shown in \cite{abbasi2011improved} (specifically, Theorem 2) that for an arm $a$ and any constant $\delta\in(0,1)$, with probability at least $1-\delta$,  
\begin{align}
	||\hat{\theta}_{a}-\theta_a||_{A_{a}(t-1)}\le 1+\sqrt{d\log\left(\frac{1+t}{\delta}\right)}.
\end{align}
Therefore, if we choose $\delta = \delta_0/K$ and $\alpha>\sqrt{2d\log(KT/\delta_0)}$,
then with probability at least $1-\delta_0$, we have $\alpha>||\hat{\theta}_a-\theta_a||_{A_{a}(t-1)}$, $\forall a\in\mathcal{A}$ and $t$, and consequently, (\ref{lm1eq2}) holds with probability $1-\delta_0$. Moreover, with probability $\delta_0$ when the upper bounds in (\ref{lm1eq2}) does not hold, the cumulative regret is trivially upper bounded by $T$.

Furthermore, let $\mathcal{T}_a$ be the set of time steps when arm $a$ is selected up to time $T$, the RHS of (\ref{lm1eq2}) satisfies:
\begin{equation}
\begin{aligned}
	&\sum_{t=1}^{T}2\alpha ||x_{u_t}||_{A_{a_t}^{-1}(t-1)}\\
	= &2\alpha\sum_{a\in\mathcal{A}}\sum_{t\in\mathcal{T}_a}||x_{u_t}||_{A_{a}^{-1}(t-1)}\\
	\le & 2\alpha\sum_{a\in\mathcal{A}}\sqrt{|\mathcal{T}_a|\sum_{t\in\mathcal{T}_a}||x_{u_t}||^2_{A_{a}^{-1}(t-1)}}\\
	\le &2\alpha\sum_{a\in\mathcal{A}}\sqrt{|\mathcal{T}_a|\cdot 2d\log \left(1+\frac{|\mathcal{T}_a|}{d}\right)}\\
	\le & 2\alpha\sqrt{2TdK\log\left(\frac{T}{d}\right)},
\end{aligned}
\end{equation}
where the first and third inequalities hold by Cauchy-Schwarz inequality and the second  inequality hold by Lemma 11 in \cite{abbasi2011improved} and Lemma 3 in \cite{guo2019adalinucb}.
In summary, the expected cumulative regret under the stationary environment is upper bounded by
\begin{equation}
\begin{aligned}
	\mathbb{E}[R(T)]\le &T\cdot\mathbb{P}(\tau_1\le T) +  K+T(\gamma+\delta_0) + 2\alpha\sqrt{2TdK\log\left(\frac{T}{d}\right)}.
\end{aligned}
\end{equation}

\subsection{Proof of Lemma 2}\label{pflm2}
Define $\tau_{a,1}$ be the first detection time of arm $a$. Then $\tau_1=\min_{a\in\mathcal{A}}\{\tau_{a,1}\}$ and
\begin{align}
	\mathbb{P}(\tau_1\le T)\le\sum_{a\in\mathcal{A}}\mathbb{P}(\tau_{a,1}\le T).
\end{align}
Let $\{(x_i,r_{a,i})\}_{i=t-\omega+1,..,t}$ be the last $\omega$ observations of arm $a$ before time $t$ and define
\begin{align}
	S_{a,t}=\frac{2}{\omega}\left|\sum_{i=t-\omega/2+1}^{t}x_i^T\tilde{\theta}_a(t-\omega+1, t-\omega/2)-r_{a,i}\right|,
\end{align}
where $\tilde{\theta}_a(t-\omega+1,t-\omega/2)$ is the estimate of $\theta_a$ based on the observations in $\{(x_i,r_{a,i})\}_{i=t-\omega+1}^{t-\omega/2}$. According to the modified PSLinUCB algorithm, we have
\begin{align}
\tau_{a,1}=\inf\{t\le \omega:S_{a,t}\ge b+c\}.
\end{align}
Moreover, we define $\tau_{a,1}^{(j)}=\inf\{t=j+n\omega,n\in\mathbb{Z}^+:S_{a,t}\ge b+c\}$. Then it is not difficult to see that at each $t_n=j+n\omega,n\in\mathbb{Z}^+$, the observations used for change detection are disjoint and thus, $\tau_{a,1}^{(j)}$ is a random variable with the geometric distribution:
\begin{align}
	\mathbb{P}(\tau_{a,1}^{(j)}=n\omega+j)=p(1-p)^{n-1},
\end{align}
where $p=\mathbb{P}(S_{a,\omega}>b+c)$ and thus
\begin{align}
	\mathbb{P}(\tau_{a,1}\le T)\le \omega(1-(1-p)^{T/\omega}).
\end{align}
To upper bound $p$, we have
\begin{equation}\label{lm2eq1}
\begin{aligned}
	&\mathbb{P}(S_{a,\omega}> b+c)\\
	\le& \mathbb{P}\left(\frac{2}{\omega}\left|\sum_{i=t-\omega/2+1}^{t}x_i^T\tilde{\theta}_a(t-\omega+1, t-\omega/2)-x_i^T\theta_a\right|>b\right)\\
	&+\mathbb{P}\left(\frac{2}{\omega}\left|\sum_{i=t-\omega/2+1}^{t}x_i^T\theta_a-r_{a,i}\right|>c\right)\\
\end{aligned}
\end{equation}
For the first term in the RHS of (\ref{lm2eq1}), if we choose $b$ to satisfy the following condition for any $t$:
\begin{align}\label{bat}
b\ge \sqrt{2d\log\left(\frac{\omega}{\delta_1}\right)}\left(\frac{2}{\omega}\sum_{i=t-\omega/2+1}^{t}||x_i||_{\tilde{A}^{-1}_{a}(t-\omega+1,t-\omega/2)}\right)
\end{align}
where $\tilde{A}_a(t-\omega+1,t-\omega/2)=\sum_{i=t-\omega+1}^{t-\omega/2}x_ix_i^T$, then the first term in the RHS of (\ref{lm2eq1}) is upper bounded by $\delta_1$ according to Lemma 2 in \cite{guo2019adalinucb} and Theorem 1 in \cite{abbasi2011improved}. The second term can be bounded by $2\exp(-\omega c^2)$ according to the Hoeffding's inequality. Let $\delta_1=1/(2T^2)$ and $c\ge \sqrt{\frac{2}{\omega}\log (2T)}$, it is not difficult to see that 
\begin{align}
p=\mathbb{P}(S_{a,\omega}> b+ c)\le T^{-2}
\end{align}
Since $(1-x)^{a}>1-ax$ for $a>1$ and $x\in(0,1)$, it can be shown that
\begin{align}
	\mathbb{P}(\tau_1\le T)\le\sum_{a\in\mathcal{A}}\mathbb{P}(\tau_{a,1}\le T)\le KT^{-1}.
\end{align}
\subsection{Proof of Lemma 3}
Notice that the round-robin exploration in the algorithm guarantees that within $L/2$ time steps, each arm is sampled at least $\omega/2$ times. We upper bound the probability of $\{\tau_1>\nu_1+L/2\}$ as follows:  consider $a$ be the arm at which the change point occurs. Let $t$ be the time step when $a$ is sampled $\omega/2$ times in the new stationary segment (notice that $t\le\nu_1+L/2$). The change at $a$ is not detected only if one of the following events happens: 
\begin{align}
	&E_1=\left\{\frac{2}{\omega}\left|\sum_{i=t-\omega/2+1}^{t}x_i^T\tilde{\theta}_{a}(t-\omega+1, t-\omega/2)-x_i^T\theta_{a}^{\textrm{old}}\right|>b\right\},\\\
	&E_2=\left\{\frac{2}{\omega}\left|\sum_{i=t-\omega/2+1}^{t}x_i^T\theta_{a}^{\textrm{new}}-r_{a,i}\right|>c\right\},\\
	&E_3=\left\{\frac{2}{\omega}\left|\sum_{i=t-\omega/2+1}^{t}x_i^T\theta_{a}^{\textrm{old}}-x_i^T\theta_{a}^{\textrm{new}}\right|<b+c\right\},
\end{align}
where $\theta_a^{\textrm{new}}$ and $\theta_a^{\textrm{old}}$ correspond to the ground-truth preference vectors of arm $a$ after and before the change point. Therefore, 
\begin{equation}
\mathbb{P}(\tau_1>\nu_1+L/2)\le\mathbb{P}({E}_1)+\mathbb{P}({E}_2)+\mathbb{P}({E}_3)
\end{equation}
The first two terms has been shown to be upper bounded by $1/T^2$ in the proof of Lemma 2 and the last term equals $0$ under the condition that $\Delta_{a}^{(1)}(x)\ge b+c$ for any $x$. Therefore, the conclusion in Lemma 3 holds.
\subsection{Additional Numerical Results}

\subsubsection{Regret Analysis on Synthetic Datasets}
We use synthetic datasets to evaluate the regret performance of the proposed learning algorithms. In the first experiment, we generate a dataset under the disjoint payoff model. Specifically, we assume a time horizon of length $T=20000$. We randomly generate $K=10$ arms. Each arm $a$ is associated with a $m$-dimensional ($m=5$) feature vector $y_a$ with $||y_a||_2\le 1$. We consider a single user setting where a user $u$ is associated with a $d$-dimensional ($d=5$) feature vector $x_u$ with $||x_u||_2\le 1$. The $d$-dimensional preference vectors $\theta_a(t),\forall a$ are randomly generated satisfying the piecewise-stationary assumption (the preference vector $\theta_a(t)$ changes every $2000$ time steps) and $||\theta_a(t)||_2\le 1$. The reward of playing an arm $a$ at time $t$ is generated according to the disjoint payoff model, i.e., $r_a(t)=x_u^T\theta_a(t)+\epsilon$, where $\epsilon$ is a Gaussian noise with $\mu=0$ (mean) and $\sigma=0.2$ (standard deviation). 

We compare the cumulative regret of PSLinUCB-Disjoint and LinUCB-Disjoint. To guarantee a fair comparison, the parameters $\alpha$ balancing the tradeoff between exploration and exploitation in the UCB indices of the two algorithms are equal ($\alpha=1$).  In PSLinUCB-Disjoint, we set $\omega=100$ and $\delta=0.35$. The experiment is run 100 times and the simulation results are included in Fig. \ref{fig:syn_disjoint}. It can be seen that the PSLinUCB-Disjoint algorithm adapts to the changing environment and achieves a lower cumulative regret ($30\%$ performance gain).
\begin{figure}[h!]
	\begin{center}
	\includegraphics[scale = 0.32]{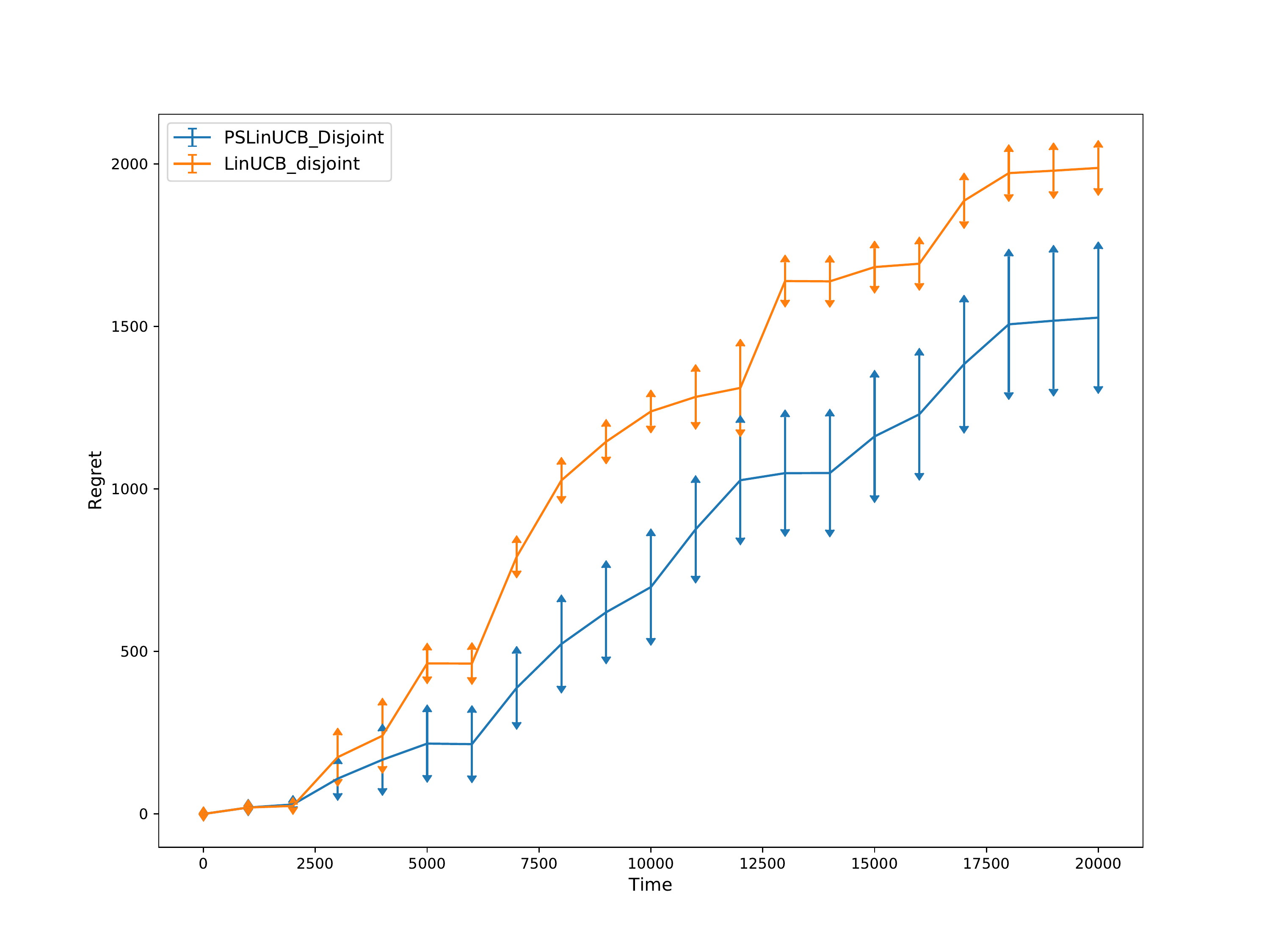}
		\caption{Regret v.s. time under the disjoint payoff model.}\label{fig:syn_disjoint}
	\end{center}
\end{figure}

In the second experiment, we consider the hybrid payoff model. In addition to the parameters generated in the first experiment, we further construct an $m\times d$-dimensional joint preference vector $\beta$. The random reward of playing an arm $a$ at time $t$ is generated according to the hybrid payoff model, i.e., $r_a(t)=x_u^T\theta_a(t)+z_{u,a}^T\beta+\epsilon$, where $z_{u,a}=\textrm{vec}(x_uy_a^T)$ and $\epsilon$ is a Gaussian noise with $\mu=0$ and $\sigma=0.2$. We compare the regret performance of PSLinUCB-Hybrid and LinUCB-Hybrid with $\alpha=1.5$. In PSLinUCB-Hybrid, we set $\omega=100$ and $\delta=0.4$. The experiment is also run 100 times and the simulation results are included in Fig. \ref{fig:syn_hybrid}. Similar performance gain can be observed.
\begin{figure}[h!]
	\begin{center}
	\includegraphics[scale = 0.32]{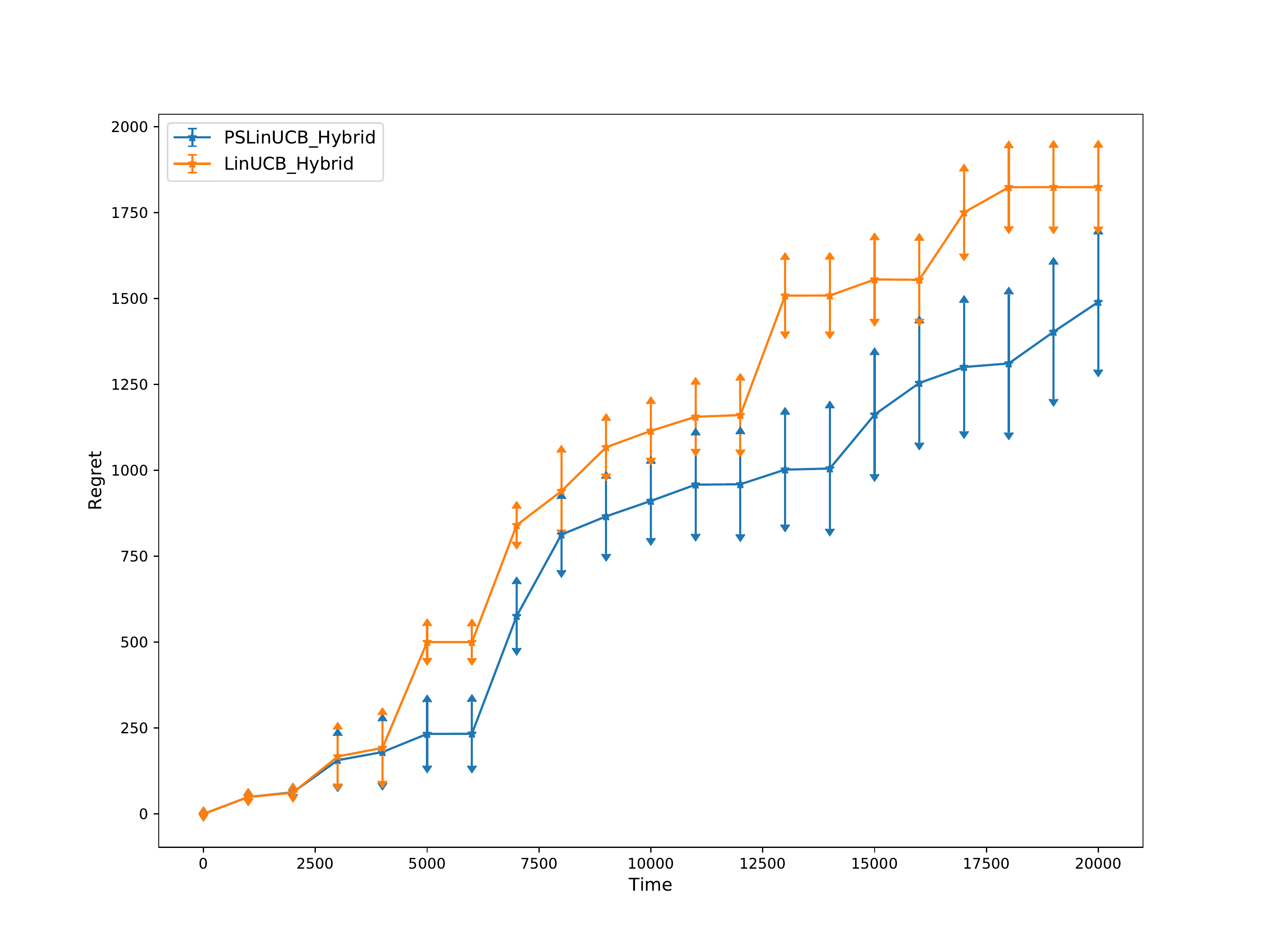}
		\caption{Regret v.s. time under the hybrid payoff model.}\label{fig:syn_hybrid}
	\end{center}
\end{figure}
\subsubsection{Recommendation Performance on Real Datasets}

The detailed recommendation performance (i.e., CTR) of the proposed algorithms along with baseline ones on the Yahoo! dataset are summarized in Table \ref{tab:yahoo}. In PSLinUCB-Disjoint, we set $\alpha=0.2$, $\omega=1000$, and $\delta=0.025$. In PSLinUCB-Hybrid, we set $\alpha=0.15$, $\omega=1200$, and $\delta=0.03$. In addition to the comparison results discussed in the main file, PSLinUCB-Disjoint and PSLinUCB-Hybrid achieves a performance gain of $59.2\%$ and $61.2\%$ compared with the Random policy, which does not learn from the observation history.

\begin{table*}[h!]
\begin{center}
\begin{tabular}{|c|c|c|c|}
\hline
\multicolumn{2}{|c|}{Stationary} & \multicolumn{2}{c|}{Non-Stationary}   \\ \hline
Algorithm           & CTR        & Algorithm         & CTR                      \\ \hline
Random              & 0.03541    & /                 & /                         \\ \hline
UCB                 & 0.04002    & MUCB              & 0.04058                 \\ \hline
LinUCB-uniform      & 0.04121    & DenBand           & 0.04353                \\ \hline
LinUCB-Disjoint     & 0.05491    & PSLinUCB-Disjoint & \textbf{0.05639}    \\ \hline
LinUCB-Hybrid       & 0.05638    & PSLinUCB-Hybrid   & \textbf{0.05711}        \\ \hline
\end{tabular}
\caption{Comparison of CTR on Yahoo dataset.}\label{tab:yahoo}
\end{center}
\end{table*}

In the second experiment on LastFM, the pre-processing step on the dataset, although follows the same strucutre, is different in details from the one adopted in \cite{wu2019dynamic}. Therefore, the DenBand algorithm does not achieve the expected performance as claimed in \cite{wu2018learning}. In Table \ref{tab:lastfm}, we only report the simulation results of the proposed algorithms and their corresponding opponents in the stationary setting. Note that in PSLinUCB-Disjoint, $\alpha=0.15$, $\omega= 1200,\delta = 0.035$. In PSLinUCB-Hybrid, $\alpha=0.2$, $\omega=1000$, $\delta=0.02$.
\begin{table*}[h!]
\begin{center}
\begin{tabular}{|c|c|c|c|}
\hline
\multicolumn{2}{|c|}{Stationary} & \multicolumn{2}{c|}{Non-Stationary}   \\ \hline
Algorithm           & CTR        & Algorithm         & CTR                      \\ \hline
LinUCB-Disjoint     & 0.03341    & PSLinUCB-Disjoint & 0.03408   \\ \hline
LinUCB-Hybrid       & 0.04046    & PSLinUCB-Hybrid   & \textbf{0.04143}        \\ \hline
\end{tabular}
\caption{Comparison of CTR on LastFM dataset.}\label{tab:lastfm}
\end{center}
\end{table*}
\subsubsection{Sensitivity Analysis}
At last, we test the proposed algorithms' sensitivity to hyper-parameters:  $\omega$ and $\delta$ on both the Yahoo! dataset and the LastFM dataset. Since the effect of users' changing interests on the recommendation performance emerges after a sufficient time of learning, we use the first $1/2$ of the Yahoo dataset and the entire LastFM dataset for testing. From the results shown in Figure \ref{fig:sensitivity_yahoo} and Figure \ref{fig:sensitivity_lastfm}, we observe that both PSLinUCB-Disjoint and PSLinUCB-Hybrid are relatively robust towards the change of the hyper-parameter within certain ranges.
\begin{figure}[h!]
\centering
	\begin{subfigure}{.45\textwidth}
			\includegraphics[scale = 0.25]{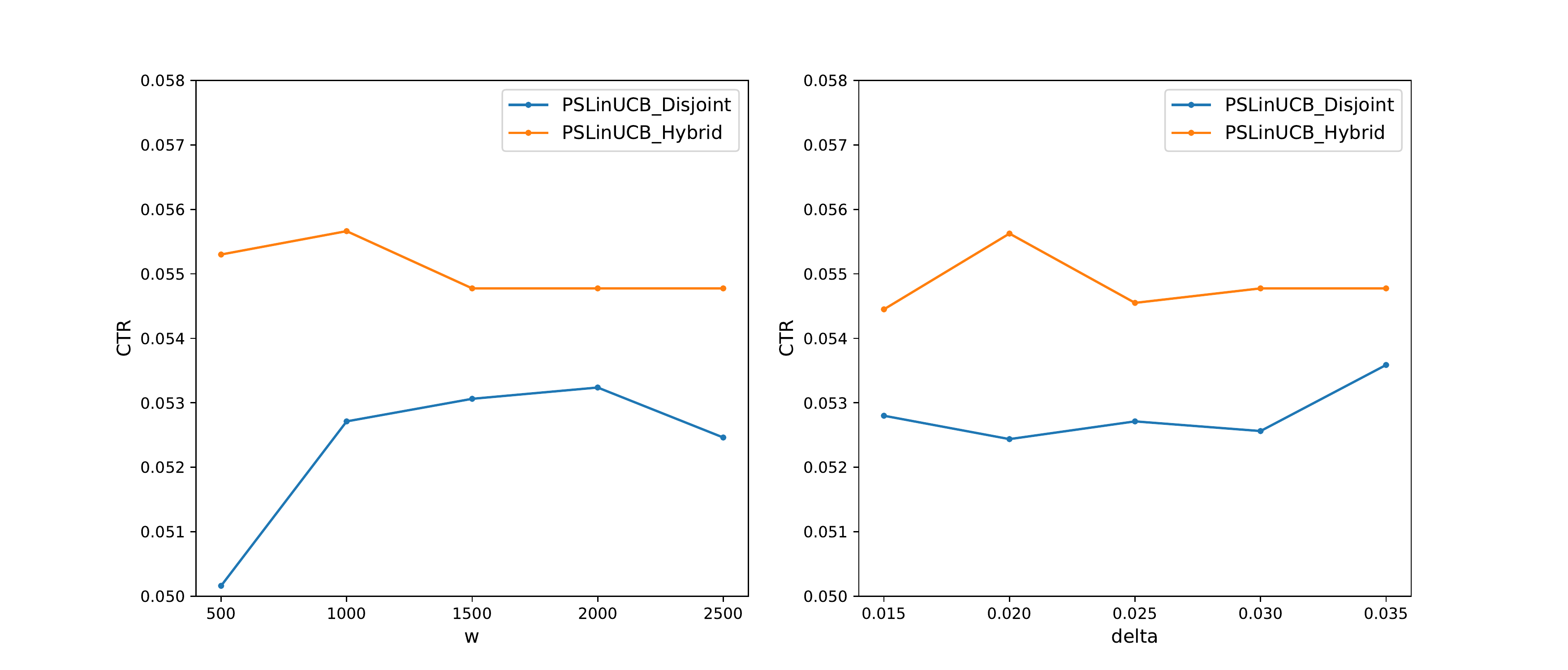}
			\caption{Yahoo! dataset.}\label{fig:sensitivity_yahoo}
	\end{subfigure}	
	\begin{subfigure}{.45\textwidth}
	\includegraphics[scale = 0.25]{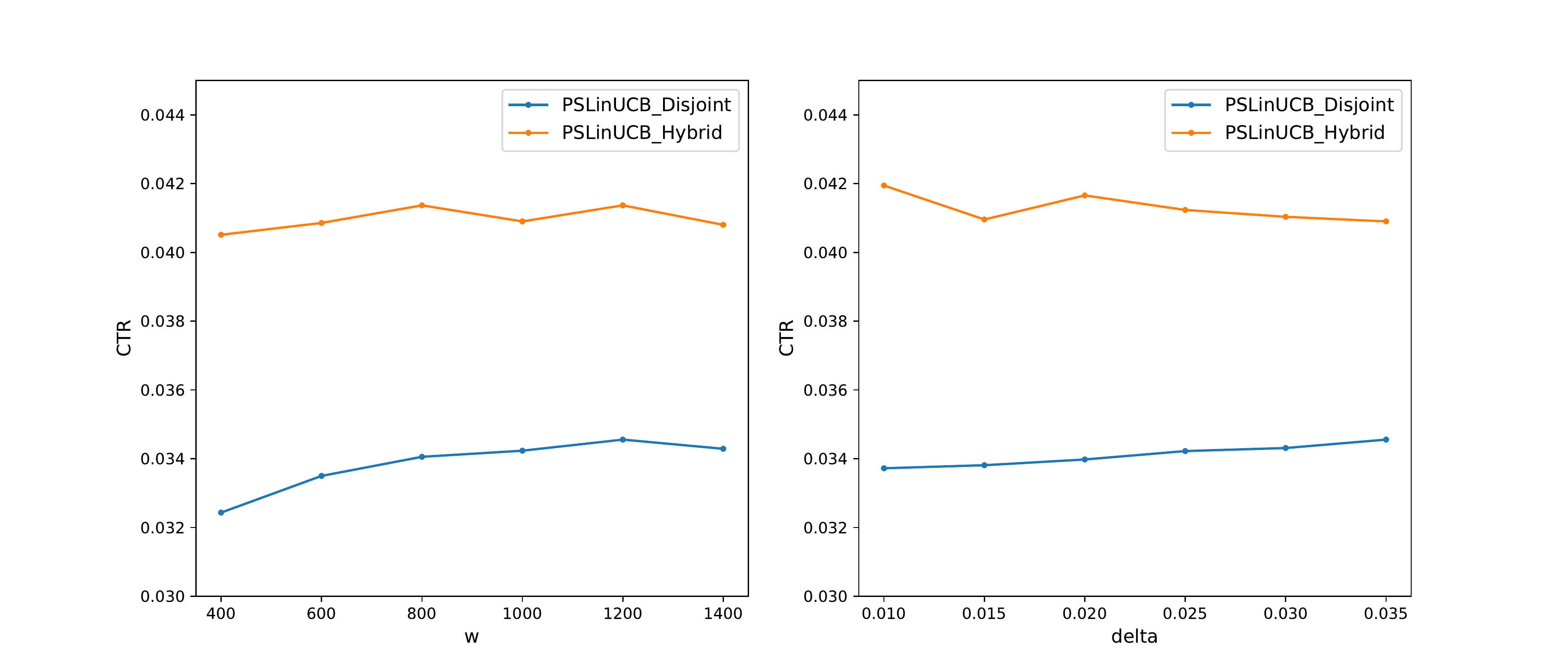}
		\caption{LastFM dataset.}\label{fig:sensitivity_lastfm}
	\end{subfigure}
	\caption{Sensitivity analysis.}
\end{figure}

\end{document}